\documentclass[final]{opt2025}

\usepackage[utf8]{inputenc} % allow utf-8 input
\usepackage[T1]{fontenc}    % use 8-bit T1 fonts
\usepackage{hyperref}       % hyperlinks
\usepackage{url}            % simple URL typesetting
\usepackage{booktabs}       % professional-quality tables
\usepackage{amsfonts}       % blackboard math symbols
\usepackage{nicefrac}       % compact symbols for 1/2, etc.
\usepackage{microtype}      % microtypography
\usepackage{xcolor,colortbl}

\usepackage{multirow}
\usepackage{amsmath}
\usepackage{braket}
\usepackage{mathrsfs}  
\usepackage{amssymb}
\usepackage{mathtools}
\usepackage{amssymb}
\usepackage{ulem}

\newcommand{\R}{\mathbb{R}}
\newcommand{\N}{\mathbb{N}}

\newtheorem{fact}[theorem]{Theorem}
\newtheorem{defn}[theorem]{Definition} 

\title{OrthoGrad Improves Neural Calibration}

\optauthor{%
\Name{C. Evans Hedges} \Email{evans.hedges@du.edu}\\
\addr University of Denver}

\begin{document}

\maketitle

\begin{abstract}

We study $\perp$Grad, a geometry-aware modification to gradient-based optimization that constrains descent directions to address overconfidence, a key limitation of standard optimizers in uncertainty-critical applications. By enforcing orthogonality between gradient updates and weight vectors, $\perp$Grad alters optimization trajectories without architectural changes. On CIFAR-10 with 10\% labeled data, $\perp$Grad matches SGD in accuracy while achieving statistically significant improvements in test loss ($p=0.05$), predictive entropy ($p=0.001$), and confidence measures. These effects show consistent trends across corruption levels and architectures. $\perp$Grad is optimizer-agnostic, incurs minimal overhead, and remains compatible with post-hoc calibration techniques.

Theoretically, we characterize convergence and stationary points for a simplified $\perp$Grad variant, revealing that orthogonalization constrains loss reduction pathways to avoid confidence inflation and encourage decision-boundary improvements. Our findings suggest that geometric interventions in optimization can improve predictive uncertainty estimates at low computational cost.

\end{abstract}

\section{Introduction}

Neural networks are increasingly deployed in settings where prediction confidence must be reliable, motivating both algorithmic and post-hoc approaches to calibration. While most intrinsic methods target loss function design or regularization, we open a new research direction: using geometric constraints on optimization trajectories to improve calibration.

Specifically, we provide the first systematic study of $\perp$Grad (read as "OrthoGrad") for calibration, which alters descent directions by projecting gradients orthogonally to layer weight vectors during training. This simple geometric constraint changes the optimization trajectory in a way that systematically reduces overconfidence without degrading accuracy. We evaluate $\perp$Grad empirically on CIFAR-10 and CIFAR-10C using ResNet18 and WideResNet-28-10, with a focus on the low-data regime. Our results show that $\perp$Grad improves calibration metrics and robustness under input corruption while remaining compatible with standard post-hoc calibration techniques.

Theoretically, we prove convergence of a simplified $\perp$Grad variant and characterize its fixed points in positive homogeneous networks. These results suggest a mechanism by which $\perp$Grad prevents loss reduction via confidence scaling alone, encouraging decision-boundary improvements instead. Together, our findings show that geometry-aware optimization can enhance calibration.

\section{Background}

A classifier is said to be \emph{calibrated} when predicted confidence scores reflect the true likelihood of correctness. This property is critical when predictive uncertainty informs downstream decisions. Guo et al.\ introduced temperature scaling in \cite{guo2017calibration}, demonstrating that modern neural networks are often poorly calibrated despite high accuracy.

Calibration techniques fall into two categories. \emph{Intrinsic} methods improve calibration during training, e.g., via loss modifications \cite{kumar2018trainable,mukhoti2020calibrating}, data augmentation, or mixup \cite{thulasidasan2019mixup}. \emph{Post-hoc} methods, including temperature scaling, Platt scaling, and isotonic regression, adjust model outputs after training without changing weights. Post-hoc methods depend on held-out validation data and cannot correct uncertainty misestimation rooted in model internals.

Recent work has explored orthogonality to improve deep networks, primarily for generalization and stability: orthogonal convolutional filters \cite{wang2020orthogonal}, inter-layer gradient orthogonality \cite{xie2017all,tuddenham2022orthogonalising}, and orthogonalized descent to accelerate grokking \cite{prieto2025grokking}. We extend this line to calibration, studying $\perp$Grad, which projects layer-wise gradients orthogonally to their corresponding weights at each update. While $\perp$Grad was originally proposed to stabilize training near grokking \cite{prieto2025grokking}, we examine its effects on model calibration under limited data and distribution shift.

To our knowledge, no prior work has investigated the link between orthogonalized gradient updates and calibration. We provide the first empirical evaluation of whether geometry-aware descent can improve metrics such as expected calibration error (ECE) and predictive entropy without harming accuracy.

\section{Theoretical Analysis} 

We first formally define the $\perp$Grad algorithm for continuous loss functions on $\R^p$. There are two important variations: the practical implementation used in our experiments (with renormalization) and a simplified variant (without renormalization) for which we can prove convergence. The following definition is the practical implementation used in our experiments and in \cite{prieto2025grokking}, and the definition for the simplified variant is found in the appendix. 

\begin{defn} For a differentiable loss function $L: \R^p \rightarrow \R$, learning rate $\eta > 0$, numerical stability constant $\epsilon > 0$, we define the $\perp$Grad update procedure as follows: 
\begin{enumerate}
\item Begin with $\theta \in \R^p$, 
\item Next let $g = \nabla L(\theta) - \frac{\braket{\nabla L(\theta), \theta}}{||\theta||^2} \theta$. This is the orthogonalized gradient. 
\item If $||g|| = 0$, we do not perform an update and return $\theta$. Otherwise we continue. 
\item We now renormalize the gradient: 
\[ 
\hat{g} = \left( \frac{||\nabla L(\theta)||}{||g|| + \epsilon} \right) g. 
\] 
\item Finally, we update $\theta' = \theta - \eta \hat{g}$. 
\end{enumerate}
\end{defn}

Note that we re-normalize the orthogonalized gradient to have the same magnitude as the original gradient (with a slight modification for numerical stability). Unfortunately this renormalization leads to slightly less desirable theoretical properties. In particular, if $\perp$Grad with renormalization converges, it converges to a stationary point for $L$, but we cannot guarantee convergence itself.

However, without renormalization, we can prove convergence under standard assumptions:

\begin{theorem}\label{convergence} Suppose $L: \R^p \rightarrow \R$ is bounded from below, differentiable, and $\nabla L$ is Lipschitz with Lipschitz constant $k$. Then for any $\eta \in (0, 1/k)$, and any initialization $\theta_0 \in \R^p$, $\perp$Grad (without gradient renormalization) will converge to some $\theta^*$ satisfying: 
\[
\left| \braket{\nabla L(\theta^*), \theta^*} \right| = ||\theta^*|| \cdot ||\nabla L(\theta^*)||. 
\]
In particular, $\nabla L(\theta^*)$ is parallel to $\theta^*$.
\end{theorem}

The proof (see Appendix) follows standard arguments. While we describe orthogonalization at the model level for clarity, the result extends naturally to layer-wise orthogonalization with only notational modifications.

This implies that $\perp$Grad converges to points where further loss reduction requires rescaling model weights and biases. As noted in \cite{prieto2025grokking}, for positive homogeneous networks this corresponds to increasing confidence rather than altering the decision boundary. Hence, $\perp$Grad reaches stationary points w.r.t.\ the decision boundary and avoids reducing loss by simply inflating confidence.

For the renormalized variant used in our experiments, convergence guarantees do not hold, but if it converges, it must converge to a stationary point. This represents a theory-practice gap in our work: while we prove convergence for the non-renormalized case, experiments use renormalization for numerical stability. However, orthogonal updates fundamentally constrain trajectories regardless of renormalization: in positive homogeneous networks, $\perp$Grad still encourages decision-boundary improvements over confidence scaling.

Empirically, we observed no significant performance differences between renormalized and non-renormalized variants in a 10-seed comparison on CIFAR-10 and CIFAR-10C, suggesting renormalization preserves the core geometric benefits while maintaining practical stability. The constraint structure remains: gradient components parallel to weights are removed, forcing optimization along decision-boundary-improving directions.

\section{CIFAR-10 Results}

\subsection{Training on CIFAR-10} 

Using 20 different seeds, we selected a random $10\%$ of the training dataset, for a total of $500$ images per class. For each seed, we trained a ResNet18 model (modified to fit the CIFAR-10 dataset) for 100 epochs with a learning rate of $0.01$, momentum set to $0.9$, and weight decay at $5e-4$. We used a batch size of $64$ and added random flips and crops for data augmentation. The base optimizer was PyTorch's SGD, which we compared to $\perp$Grad following the implementation in \cite{prieto2025grokking}. Orthogonalization is applied layer-wise: each layer's gradient is projected orthogonally to that layer's weight vector. Note that this implementation includes gradient renormalization; while this variant does not enjoy the convergence guarantees we prove, we include it here for continuity with prior work and to isolate the effect of orthogonalization on calibration metrics. The average results across the 20 runs are shown in Table 1 and reliability diagrams can be found in the appendix. 

\begin{table}[h!]
\begin{center}
\begin{tabular}{ |c|c|c|c|c|c| } 
\hline
& SGD & $\perp$Grad & Effect Size & $95\%$ Confidence Interval & $p$ value \\ 
\hline
Top1 Accuracy & 75.18 & {\bf 75.27} & -0.05 & $(-0.67, 0.57)$ & 0.86 \\ 
Top5 Accuracy & 97.67 & {\bf 97.81} & -0.35 & $(-0.97, 0.28)$ & 0.28 \\
Loss & 1.26 & {\bf 1.19 } & 0.64 & $(0.005, 1.28)$ & 0.05 \\
ECE & 0.168 & {\bf 0.161} & 0.48 & $(-0.15, 1.11)$  & 0.14 \\
Brier Score & 0.408 & {\bf 0.400} & 0.28 &  $(-0.34, 0.91)$  & 0.37 \\
Entropy & 0.208 & {\bf 0.224 } & -1.11 & $(-1.77, -0.44)$  & 0.001 \\
Max Softmax & 0.920 & {\bf 0.914 } & 1.06 & $(0.40, 1.72)$ &  0.002 \\
Max Logit & 13.58 &{\bf 13.03 } & 1.52 & $(0.82, 2.22)$ &  $2.5 \times 10^{-5}$ \\
Logit Variance & 45.73 & {\bf 42.30 } & 2.00  & $(1.25, 2.77)$  & $2 \times 10^{-7}$ \\
% Weight Norm & 7.82 & 7.70 \\  
\hline
\end{tabular}
\end{center}
\caption{\textbf{CIFAR-10 test results across 20 seeds comparing SGD and $\perp$Grad.}
Accuracy remains unchanged, but $\perp$Grad consistently improves loss, entropy, and softmax/logit statistics.
These differences suggest improved calibration and reduced overconfidence under $\perp$Grad.
Bold indicates better performance (higher accuracy, entropy; lower loss, ECE, etc.), regardless of statistical significance.}
\label{table:1}
\end{table}

This experiment showed no significant difference in accuracy metrics. However, $\perp$Grad achieved statistically significant improvements in test loss ($p=0.05$), predictive entropy ($p=0.001$), and confidence measures: lower maximum softmax ($p=0.002$), maximum logit ($p=2.5 \times 10^{-5}$), and logit variance ($p=2 \times 10^{-7}$). 

While not reaching statistical significance, $\perp$Grad showed consistent trends toward improved ECE and Brier Score, and better confidence-correctness correlation (0.467 vs 0.445). The robustness of significant effects across multiple confidence-related metrics suggests systematic reduction in overconfidence without accuracy degradation. 

While orthogonalization has been hypothesized to introduce implicit regularization by reducing weight norm growth, our experiments do not support this effect. The final weight vector norms did not differ significantly between the two optimization methods ($79.69$ for SGD compared to $79.72$ for $\perp$Grad, $p=0.36$). 

\subsection{Temperature Scaling} 

Next, we evaluated the impact of temperature scaling on model calibration between the two optimizer choices. There was a significant difference ($p=0.003$) between optimal temperatures, with SGD requiring higher temperature scaling ($T=2.80$) compared with $\perp$Grad ($T=2.66$). However, there was no difference between the temperature scaled ECE or Brier scores (see appendix). 

Notably this means that, unlike the results in \cite{wang2021rethinking}, $\perp$Grad appears to remain amenable to post-hoc calibration and is able to improve loss and entropy by instead optimizing the decision boundary without allowing for naive scaling of outputs. Additionally, the fact that $\perp$Grad required a significantly lower temperature for calibration further indicates that $\perp$Grad converges to better calibrated models without sacrificing accuracy. 

\subsection{CIFAR-10C Evaluation}

Finally, we turn to examining how the resulting models behaved under input corruption using the CIFAR-10C dataset. We found that $\perp$Grad maintained calibration and loss improvements across corruption types and severity levels. We observed similar results to the clean experiment, with negligible differences between SGD and $\perp$Grad in accuracy metrics, but the effects on loss, entropy, max softmax/logit values, and logit variance persisted (although diminished in statistical significance). 

\begin{figure}[h!]
\centering
\includegraphics[width=1\textwidth]{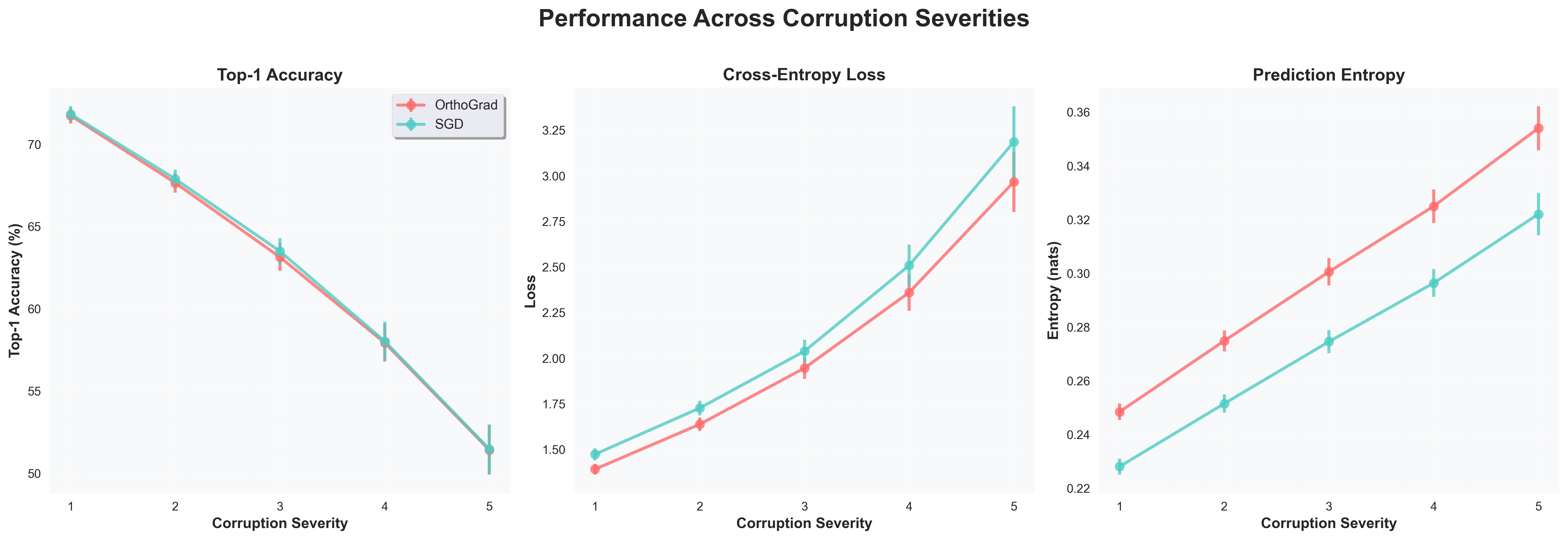}
\caption{\textbf{Comparative trends across CIFAR-10C corruption levels.}
$\perp$Grad consistently shows better loss and predictive entropy across corruption levels without sacrificing accuracy,
indicating improved robustness under input noise.}
\label{fig:combined_metrics}
\end{figure}

In all, it appears that orthogonalizing gradients had no meaningful impact on accuracy, yet it improved the model's calibration by decreasing loss and confidence and increasing entropy. 

\section{Additional Empirical Results} 

\subsection{Extended Training Results}

To examine calibration and robustness under extreme overfitting, we extended ResNet18 training to 1000 epochs with all other hyperparameters fixed. This stress test was designed to reveal optimizer behavior beyond the typical training horizon. Early stopping was not used. Due to computational constraints, results are from a single seed and should not be interpreted statistically, though they are consistent with our short-run multi-seed findings.

SGD achieved higher clean test accuracy ($70.5\%$ vs.\ $65.8\%$ for $\perp$Grad), but $\perp$Grad consistently outperformed under corruption: from level 2 onward it yielded better loss and ECE, and from level 3 onward better Top1 accuracy. Overall average accuracy across CIFAR-10C was also higher for $\perp$Grad ($60.4\%$ vs.\ $59.0\%$). Detailed comparisons appear in the appendix.

At corruption level 5, the overfit $\perp$Grad model outperformed not only the overfit SGD model but also every seed from the 100-epoch experiments. Assessing the statistical robustness of this behavior is left to future work.

\subsection{WideResNet-28-10}

We additionally ran a 5 seed experiment using WideResNet-28-10. All other hyperparameters were kept the same as the original ResNet18 experiment in Section 4. The results confirm that calibration benefits generalize across architectures: $\perp$Grad again achieved significant improvements in loss ($p=0.004$), entropy ($p=2 \times 10^{-4}$), and confidence measures while maintaining accuracy parity. The consistency of these effects across ResNet18 and WideResNet-28-10 suggests the geometric constraint addresses a fundamental optimization bias toward overconfidence that is architecture-independent. These results persisted under corruption. More details appear in the appendix. 

\section{Discussion} 

This work demonstrates that gradient orthogonalization via $\perp$Grad systematically improves neural calibration without accuracy loss. The method achieves statistically significant improvements in test loss, entropy, and confidence measures across architectures (ResNet18, WideResNet-28-10) and conditions (clean data, corruption, extended training). As an optimizer-agnostic intervention with minimal overhead, $\perp$Grad offers a practical approach for applications requiring reliable uncertainty estimates.

Key limitations include restriction to CIFAR-10/CIFAR-10C and the theory-practice gap between our convergence proof (non-renormalized) and experimental implementation (renormalized). The extended training results, while promising, require validation across multiple seeds. $\perp$Grad remains compatible with post-hoc calibration, mitigating concerns about regularization-calibration conflicts \cite{wang2021rethinking}.

Theoretically, we prove convergence for a simplified variant and characterize favorable stationary points where loss reduction requires decision-boundary improvements rather than confidence scaling. While our convergence proof applies to the non-renormalized case, empirical comparison shows no significant differences between variants, suggesting the geometric benefits transfer to the practical implementation.

Our results indicate that geometric constraints on optimization trajectories offer a promising direction for improving neural calibration. $\perp$Grad provides consistent calibration improvements across architectures and conditions while maintaining the simplicity and broad applicability essential for practical deployment.

\bibliography{mybib}

\newpage 

\section{Appendix}

The appendix is organized as follows. First we describe the stationary points for $\perp$Grad with renormalization. Next we formally define the $\perp$Grad algorithm without renormalization and prove Theorem 3.1, showing that in the non-renormalized case $\perp$Grad is guaranteed to converge under standard assumptions. We then discuss the stable points for $\perp$Grad (with and without renormalization), showing that in the case of positive homogenous classification networks they correspond with stationary points with respect to the decision boundary. 

\subsection{$\perp$Grad with Renormalization}

We first note that if $\perp$Grad with renormalization converges, it converges to a stationary point for $L$. Note that this does not imply that $\perp$Grad with renormalization will converge, and in fact with a fixed learning rate we suspect that in many cases it will not converge, but we leave a deeper discussion of this potential lack of convergence to future work. Here we characterize the stationary points for $\perp$Grad with renormalization. 

\begin{lemma} If $\perp$Grad with renormalization converges along the descent pathway $(\theta_n)_{k \in \N}$, then either $\braket{\nabla L(\theta_n), \theta_n} = ||\theta_n|| \cdot ||\nabla L(\theta_n)||$ for some $n \in \N$ at which point the $\perp$Grad trajectory stabilizes, or $||\nabla L(\theta_n)|| \rightarrow 0$. 
\end{lemma}

\begin{proof} First suppose $\perp$Grad converges along the descent pathway $(\theta_n)_{k \in \N}$. If we have $\braket{\nabla L(\theta_n), \theta_n} = ||\theta_n|| \cdot ||\nabla L(\theta_n)||$ for some $k$, it is easy to see that the $\perp$Grad trajectory stabilizes. We now assume that $\braket{\nabla L(\theta_n), \theta_n} \neq 0$ for all $k \in \N$. For each $k \in \N$, let $v_k, \hat{g}_k$ be as in the definition of $\perp$Grad. Since $(\theta_n)$ converges, we have 
\[ 
||\hat{g}_k|| = \frac{ ||\nabla L(\theta_n)|| \cdot ||v_k||}{||v_k|| + \epsilon} =   \frac{ ||\nabla L(\theta_n)|| }{1 + \frac{\epsilon}{||v_k||}} \rightarrow 0. 
\]
Suppose for a contradiction that $\lim \sup || \nabla L(\theta_n)|| = c > 0$. By passing to a subsequence we can assume without loss of generality that $\lim || \nabla L(\theta_n)|| = c$. Since $||v_k|| \leq || \nabla L(\theta_n)||$ by definition, we know 
\[
\lim \sup 1 + \frac{\epsilon}{||v_k||} \geq 1 + \frac{\epsilon}{c}, 
\]
and therefore
\[ 
\lim \inf \frac{ ||\nabla L(\theta_n)|| }{1 + \frac{\epsilon}{||v_k||}} \geq \frac{c}{1+\frac{\epsilon}{c}} > 0. 
\]
This contradicts that $||g_n|| \rightarrow 0$ and it must be the case that $||\nabla L(\theta_n)|| \rightarrow 0$.
\end{proof}

Notably, this means that if the $\perp$Grad procedure outlined in \cite{prieto2025grokking} converges (nontrivially), it converges to a stationary point for $L$. Combined with the fact that we cannot guarantee convergence for the renormalized version of $\perp$Grad, this may mitigate some of the theoretically proposed benefits discussed in Section 3.

\subsection{Convergence Results}

Next we define the $\perp$Grad procedure without renormalization: 

\begin{defn} For a differentiable loss function $L: \R^p \rightarrow \R$, learning rate $\eta > 0$, we define the $\perp$Grad (without renormalization) update procedure as follows: 
\begin{enumerate}
\item Begin with $\theta \in \R^p$, 
\item Next let $g = \nabla L(\theta) - \frac{\braket{\nabla L(\theta), \theta}}{||\theta||^2} \theta$. This is the orthogonalized gradient. 
\item We update $\theta' = \theta - \eta g$. 
\end{enumerate}
\end{defn}

Without renormalization, we are able to prove some desirable convergence properties. Not only that the algorithm converges under standard assumptions, but additionally the stable points have desirable properties when it comes to positive homogenous model architectures. 

{\bf Theorem 3.1} {\it Suppose $L: \R^p \rightarrow \R$ is bounded from below, differentiable, and $\nabla L$ is Lipschitz with Lipschitz constant $k$. Then for any $\eta \in (0, 1/k)$, and any non-zero initialization $\theta_0 \in \R^p$, $\perp$Grad (without gradient renormalization) will converge to some $\theta^*$ satisfying: 
\[
\left| \braket{\nabla L(\theta^*), \theta^*} \right| = ||\theta^*|| \cdot ||\nabla L(\theta^*)||. 
\]
In particular, $\nabla L(\theta^*)$ is parallel to $\theta^*$.
}

\begin{proof} 
Let $\theta_n$ be any point along the $\perp$Grad pathway and let 
\[
g_n = \nabla L(\theta_n) - \braket{ \nabla L(\theta_n), \theta_n} \frac{\theta_n}{||\theta_n||^2}
\]
denote the orthogonalized gradient. Define the update $\theta_{n+1} = \theta_n - \eta g_n$. Since $\nabla L$ is Lipschitz with constant $k$, we apply a standard descent bound:
\[
\begin{aligned}
L(\theta_{n+1}) & \leq L(\theta_n) - \eta \braket{ \nabla L(\theta_n), g_n} + \frac{\eta^2 k}{2} ||g_n||^2 \\
& = L(\theta_n) - \eta \left( ||\nabla L(\theta_n)||^2 - \frac{\braket{\nabla L(\theta_n), \theta_n}^2}{||\theta_n||^2} \right) + \frac{\eta^2 k}{2} ||g_n||^2 \\
& = L(\theta_n) - \eta ||g_n||^2 + \frac{\eta^2 k}{2} ||g_n||^2 \\
& = L(\theta_n) - \eta \left( 1 - \frac{\eta k}{2} \right) ||g_n||^2.
\end{aligned}
\]

Since $\eta < 1/k$, the coefficient $\eta \left(1 - \frac{\eta k}{2} \right)$ is positive. Therefore, the loss strictly decreases unless $g_n = 0$.Summing over $n = 0$ to $T-1$, we get:
\[
L(\theta_0) - L(\theta_T) \geq \eta \left(1 - \frac{k \eta}{2} \right) \sum_{n=0}^{T-1} ||g_n||^2.
\]

Since $L$ is bounded below by some $L_{\inf} \in \R$, we have:
\[
\sum_{n=0}^{\infty} ||g_n||^2 \leq \frac{L(\theta_0) - L_{\inf}}{\eta \left(1 - \frac{k \eta}{2} \right)}.
\]

We now note that since for each $n \in \N$, $\braket{ g_n, \theta_n} = 0$, we know 
\[ 
\sum_{k=0}^\infty || \theta_{n+1} - \theta_n||^2 = \eta^2 \sum_{n=0}^{\infty} ||g_n||^2 < \infty
\] 
and by the Cauchy criterion it must be the case that the sequence $(\theta_n)$ converges to some $\theta^*$. 

We now let 
\[ 
g^* = \nabla L(\theta^*) - \braket{ \nabla L(\theta^*), \theta^*} \frac{\theta^*}{||\theta^*||^2}
\]
By continuity of $\nabla L(\cdot)$ it is easy to see that $g^* = 0$, and the desired result follows immediately. 
\end{proof}

We note here that when it comes to performing orthogonalization for a model with parameters $\theta \in \R^p$, the proof still holds if the orthogonalization is occurring on the entire parameter set at once, or at the level of a partition of $\theta$ (in particular at the layer level). The conclusion will only differ in that each component of the stabilized gradient $\nabla L(\theta^*)$ will be parallel to the corresponding component of the vector $\theta^*$ and each component may differ by scalar multiples. For the purposes of understanding the stability of decision boundaries in positive homogenous classification models, this distinction makes no difference. 

\subsection{Decision Boundary Properties of Stable Points} 

In this section we prove that a stationary point for non-renormalized $\perp$Grad exhibits favorable properties for positive homogenous models. In particular, if the orthogonalization occurs on the level of a positive homogenous section of the network, $\perp$Grad will converge to solutions where loss cannot be improved by changing the decision boundary. 

\begin{defn} We say that for any $P \subset[1, p]$, a model $f_\theta: \R^n \rightarrow \R^k$ with parameters $\theta \in \R^p$ is {\bf positive homogenous with respect to $P$} if for any $\theta \in \R^p$ and any $c > 0$, for any $x \in \R^n$, we have 
\[ 
f_{\theta'} (x) = \alpha f_\theta(x) 
\]
where $\alpha > 0$ and $\theta' \in \R^p$ satisfying: for all $i \in P$, $\theta_i' = c \theta_i$ and for all $i \notin P$, $\theta_i' = \theta_i$. 
\end{defn}

Note here that an MLP is positive homogenous with respect to any $P$ containing the union of layers of the model. 

\begin{fact} For $\theta \in \R^p$, let $f_\theta : \R^n \rightarrow \R^k$ be a model outputting logits for a classification task. Additionally, suppose there exists a partition $\mathcal{P}$ of $[1, p]$ such that for every $P \in \mathcal{P}$, $f$ is positive homogenous with respect to $P$. For a loss function $L$ as in Theorem 3.1, if $\perp$Grad orthogonalizes gradients with respect to each $P \in \mathcal{P}$, it will converge to local minimum with respect to the decision boundary. 
\end{fact}

\begin{proof} First note that at the level of each $P \in \mathcal{P}$, by Theorem 3.1 $\perp$Grad will converge to a solution such that \[ 
\left| \braket{\nabla L(\theta_P^*), \theta_P^*} \right| = ||\theta_P^*|| \cdot ||\nabla L(\theta^*)_P||. 
\] 
It therefore follows that for the entire collection of parameters, traveling along $\nabla L(\theta^*)$ only scales the model parameters, where each collection of parameters $P \in \mathcal{P}$ will be scaled by a different magnitude. Since $f$ is positive homogenous with respect to each $P \in \mathcal{P}$, and the decision boundary is determined by the maximum output of $f_{\theta^*}$, it follows that locally improving $L$ does not change the decision boundary. 
\end{proof}

\subsection{Renormalized vs Non-Renormalized Comparison}

To address the theory-practice gap between our convergence proof (non-renormalized) and experimental implementation (renormalized), we conducted a 10-seed comparison on CIFAR-10 using identical hyperparameters. Results show no statistically significant differences, validating that renormalization preserves the core geometric benefits of orthogonalization.

\begin{table}[h!]
\begin{center}
\begin{tabular}{ |c|c|c|c|c|c|c| } 
\hline
Variant & Accuracy & Loss & ECE & Entropy & Max Softmax & Max Logit \\ 
\hline
Renormalized & 74.89 & 1.198 & 0.164 & 0.225 & 0.914 & 12.62 \\ 
Non-Renormalized & 74.14 & 1.230 & 0.167 & 0.220 & 0.916 & 12.56 \\
\hline
\end{tabular}
\end{center}
\caption{\textbf{Comparison of renormalized vs non-renormalized $\perp$Grad variants on CIFAR-10.}
Values shown are averages across 10 random seeds. No significant differences observed, indicating renormalization preserves geometric benefits while maintaining numerical stability.}
\label{table:renorm_comparison}
\end{table}

\subsection{Additional Figures} 

\begin{table}[h!]
\begin{center}
\begin{tabular}{ |c||c|c||c|c| } 
\hline
Optimizer & \multicolumn{2}{c||}{ECE} & \multicolumn{2}{c|}{Brier Score} \\
\cline{2-5}
& Before & After & Before & After \\ 
\hline
SGD & 0.168 & 0.015 & 0.041 & 0.034 \\ 
$\perp$Grad & {\bf 0.161} & 0.015 & {\bf 0.040} & 0.034 \\ 
\hline
\end{tabular}
\end{center}
\caption{\textbf{Expected Calibration Error (ECE) and Brier Score before and after temperature scaling on CIFAR-10.}
Both optimizers benefit similarly from temperature scaling, but $\perp$Grad starts with slightly better raw calibration.
This shows that $\perp$Grad is compatible with post-hoc calibration techniques, preserving gains after temperature correction.
Bold indicates better performance, regardless of statistical significance.}
\label{table:calibration}
\end{table}

\begin{figure}[h!]
\centering
\includegraphics[width=0.8\textwidth]{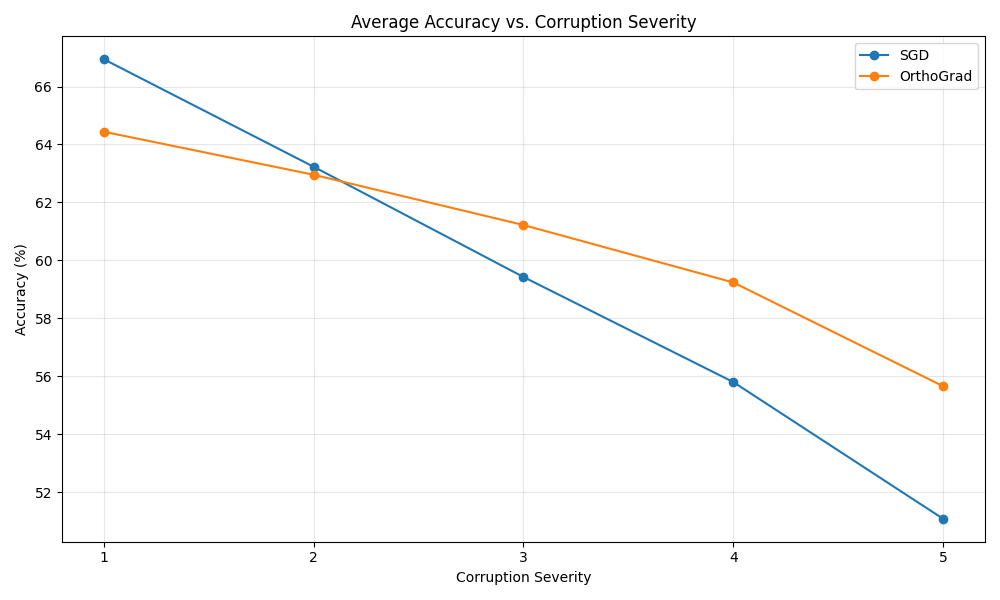}
\caption{\textbf{Overtrained ResNet-18 accuracy across CIFAR-10C corruption levels.}
In the overtrained environment accuracy initially favors SGD, however $\perp$Grad surpasses it at higher severity.}
\label{fig:extended_accuracy}
\end{figure}

\begin{figure}[h!]
\centering
\includegraphics[width=0.8\textwidth]{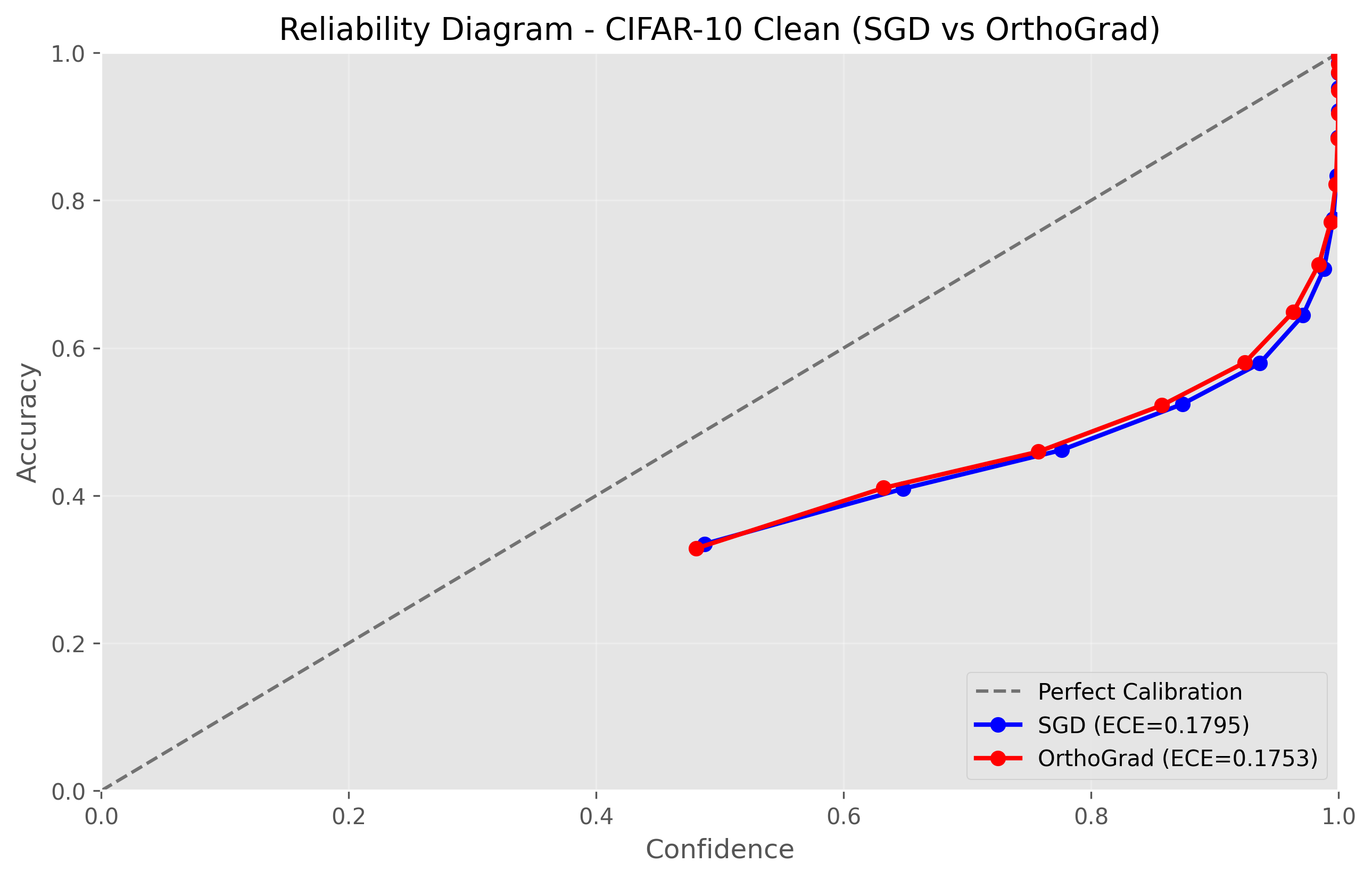}
\caption{\textbf{Reliability Diagram for ResNet18 on CIFAR-10.}
Average reliability diagram across $20$ seeds across entire CIFAR-10 test dataset. 
$\perp$Grad exhibits slightly better reliability than SGD.}
\label{fig:reliability_clean}
\end{figure}

\begin{figure}[h!]
\centering
\includegraphics[width=0.8\textwidth]{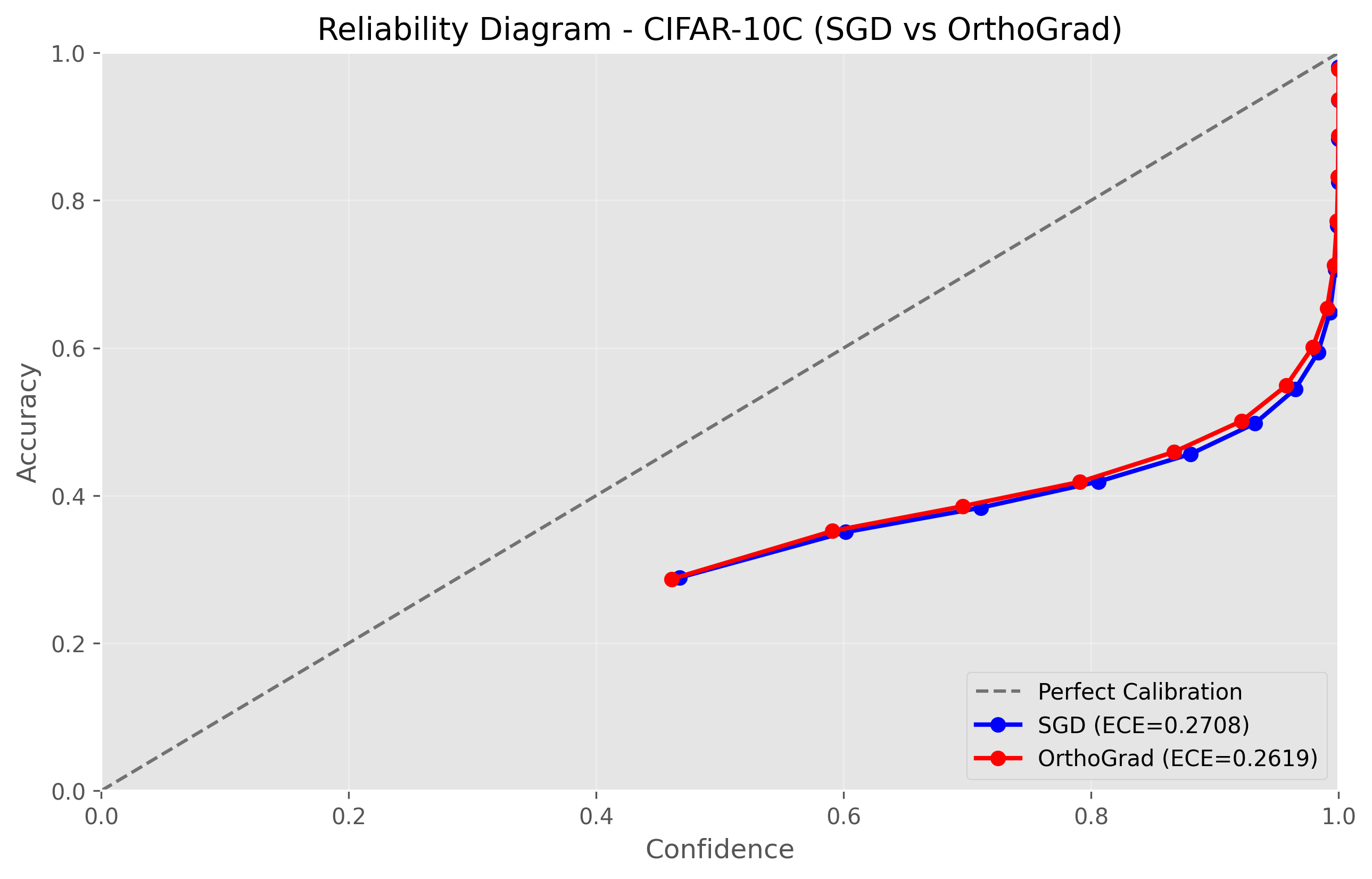}
\caption{\textbf{Reliability Diagram for ResNet18 on CIFAR-10C.}
Average reliability diagram across $20$ seeds across entire CIFAR-10C test dataset. 
$\perp$Grad exhibits slightly better reliability than SGD.}
\label{fig:reliability_corrupted}
\end{figure}

\begin{table}[h!]
\begin{center}
\begin{tabular}{|c||c|c||c|c||c|c||c|c|}
\hline
\multirow{2}{*}{Corruption Level} & \multicolumn{2}{c||}{Accuracy (\%)} & \multicolumn{2}{c||}{Loss} & \multicolumn{2}{c||}{ECE} & \multicolumn{2}{c|}{Conf-Acc Corr.} \\
\cline{2-9}
& SGD & $\perp$Grad & SGD & $\perp$Grad & SGD & $\perp$Grad & SGD & $\perp$Grad \\
\hline
1 & {\bf 67} & 64 & {\bf 1.79} & 1.91 & {\bf 0.23} & 0.24 & {\bf 0.388} & 0.382 \\
2 & 63 & 63 & 2.02 & {\bf 1.93} & 0.26 & {\bf 0.24} & 0.366 & {\bf 0.384} \\
3 & 59 & {\bf 61} & 2.62 & {\bf 2.00} & 0.29 & {\bf 0.25} & 0.347 & {\bf 0.377} \\
4 & 55 & {\bf 59} & 2.53 & {\bf 2.08} & 0.32 & {\bf 0.26} & 0.329 & {\bf 0.374} \\
5 & 51 & {\bf 55} & 2.88 & {\bf 2.24} & 0.36 & {\bf 0.28} & 0.301 & {\bf 0.343} \\
\hline
\end{tabular}
\end{center}
\caption{\textbf{Accuracy, loss, ECE, and confidence--accuracy correlation on CIFAR-10C across corruption levels (single seed, 1000 epochs).}
$\perp$Grad degrades more gracefully under corruption, outperforming SGD from corruption level 3 onward.
Calibration and loss are consistently better, even though clean accuracy is slightly lower. 
Results suggest robustness gains under overfitting conditions.
Bold indicates better performance, regardless of statistical significance.}
\label{table:combined}
\end{table}

\begin{table}[h!]
\begin{center}
\begin{tabular}{ |c|c|c|c|c|c| } 
\hline
& SGD & $\perp$Grad & Effect Size & $95\%$ Confidence Interval & $p$ value \\ 
\hline
Top1 Accuracy & 78.54 & {\bf 79.00} & -0.35 & $(-1.60, 0.90)$ & 0.59 \\ 
Top5 Accuracy & {\bf 98.05} & 97.84 & 0.50 & $(-0.76, 1.76)$ & 0.44 \\
Loss & 1.05 & {\bf 0.88 } & 2.56 & $(0.89, 4.24)$ & 0.004 \\
ECE & 0.14 & {\bf 0.12} & 2.40 & $(0.77, 4.02)$  & 0.015 \\
Brier Score & 0.35 & {\bf 0.33} & 0.86 &  $(-0.43, 2.15)$  & 0.21 \\
Entropy & 0.19 & {\bf 0.25 } & -4.18 & $(-6.40, -1.97)$  & $2 \times 10^{-4}$ \\
Max Softmax & 0.92 & {\bf 0.91 } & 2.91 & $(1.13, 4.69)$ &  0.002 \\
Max Logit & 12.11 &{\bf  8.68 } & 11.20 & $(6.14, 16.27)$ &  $3 \times 10^{-6}$ \\
Logit Variance & 31.37 & {\bf 13.83 } & 15.45  & $(8.57, 22.34)$  & $6 \times 10^{-6}$ \\
\hline
\end{tabular}
\end{center}
\caption{\textbf{CIFAR-10 WideResNet-28-10 test results across 5 seeds comparing SGD and $\perp$Grad.}
Accuracy remains unchanged, but $\perp$Grad consistently improves loss, entropy, and softmax/logit statistics.
These differences suggest improved calibration and reduced overconfidence under $\perp$Grad.
Bold indicates better performance, regardless of statistical significance.}
\label{table:wideresnet}
\end{table}

\begin{figure}[h!]
\centering
\includegraphics[width=0.8\textwidth]{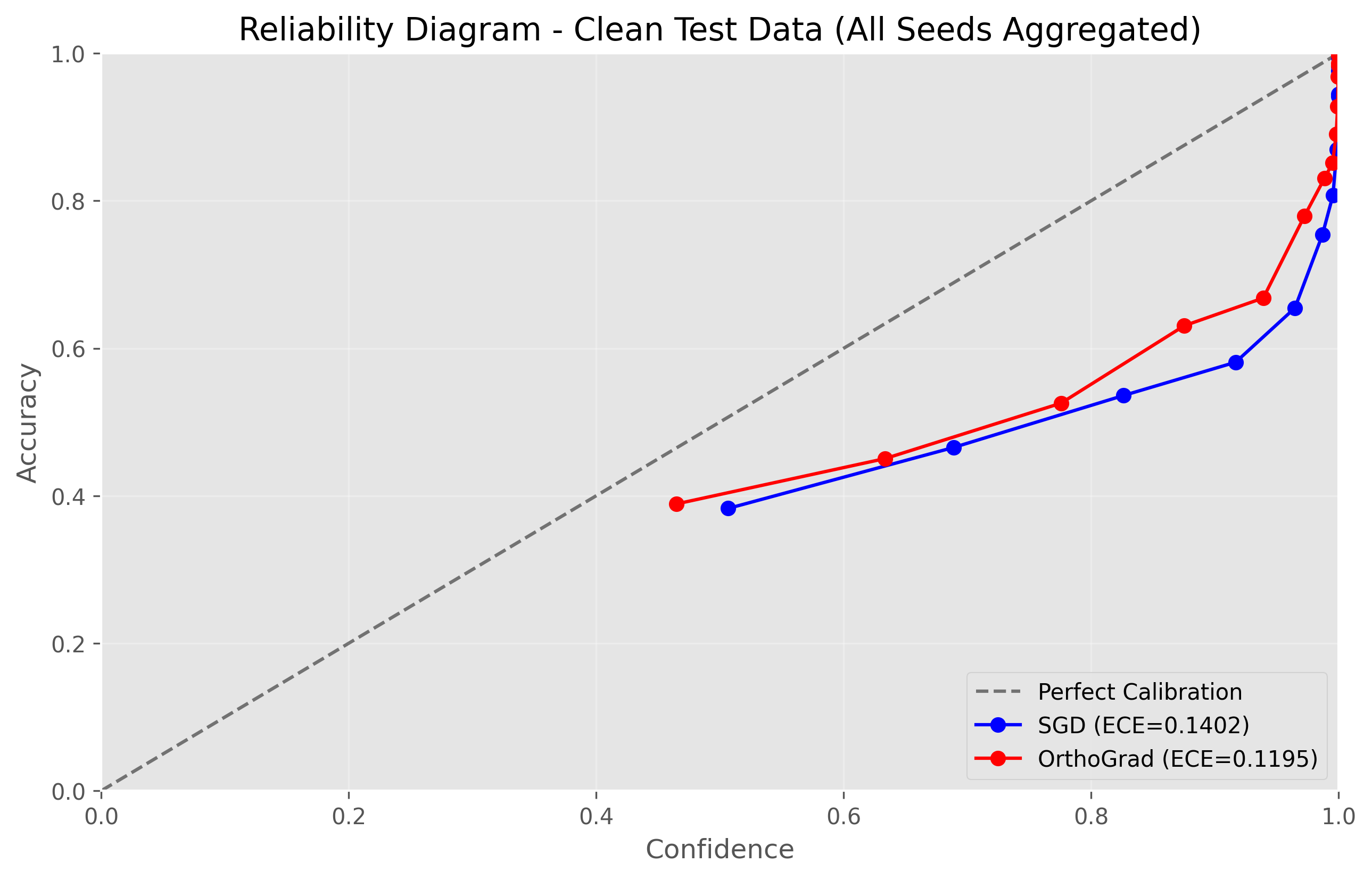}
\caption{\textbf{Reliability Diagram for WideResNet-28-10 on CIFAR-10.}
Average reliability diagram across $5$ seeds. 
$\perp$Grad exhibits consistently better reliability than SGD.}
\label{fig:wideresnet_clean}
\end{figure}

\begin{figure}[h!]
\centering
\includegraphics[width=0.8\textwidth]{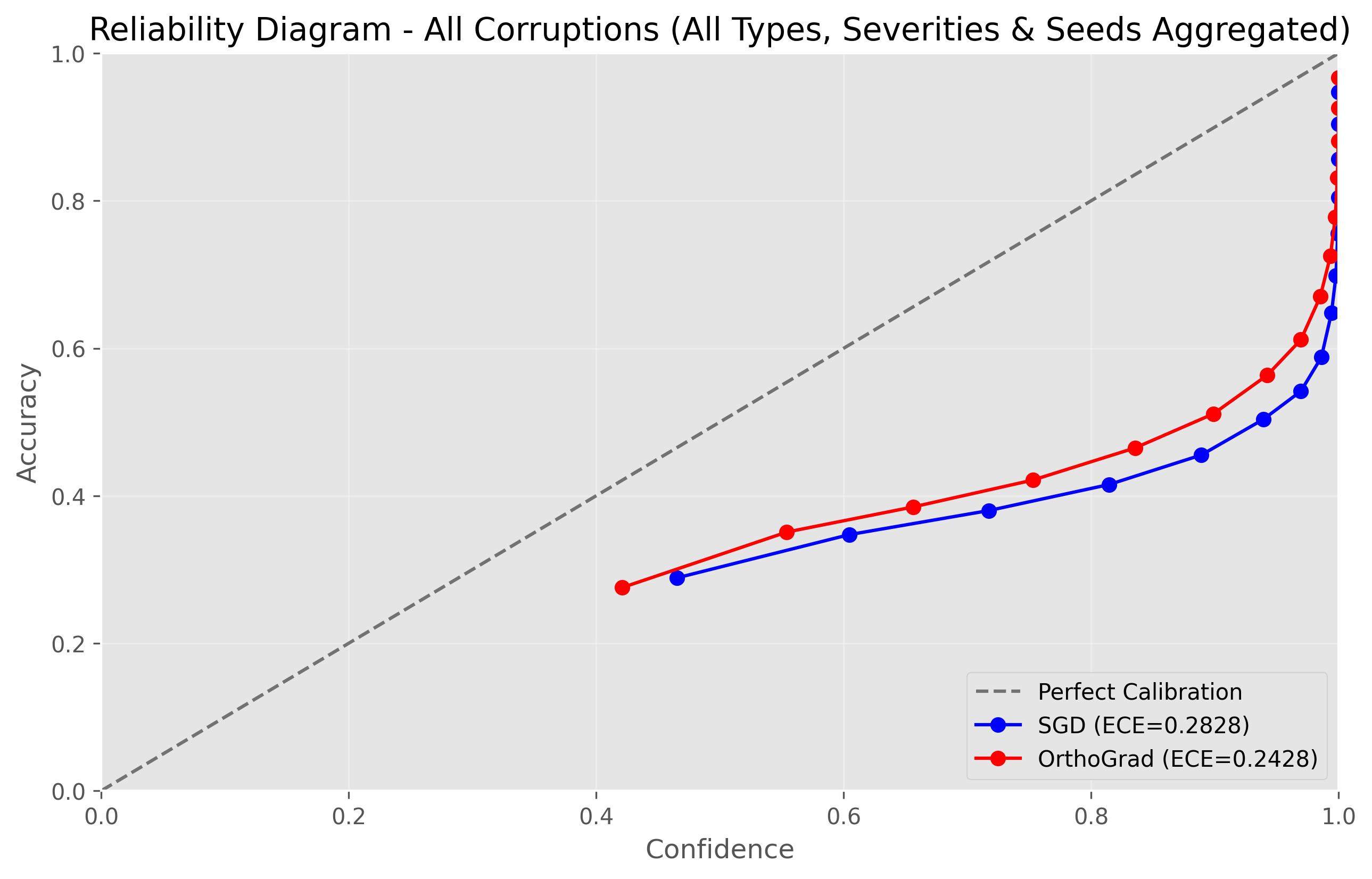}
\caption{\textbf{Reliability Diagram for WideResNet-28-10 on CIFAR-10C.}
Average reliability diagram across $5$ seeds across entire CIFAR-10C test dataset. 
$\perp$Grad exhibits consistently better reliability than SGD.}
\label{fig:wideresnet_corrupted}
\end{figure}

\end{document}